\documentclass[11pt,a4paper]{article}

\usepackage{graphicx} 

\usepackage[left=1in, right=1in, bottom=1.25in, top=1.5in]{geometry}

\usepackage{breakcites}

\usepackage{fullpage}
\usepackage{times}
\usepackage{soul}
\usepackage[utf8]{inputenc}
\usepackage[small]{caption}
\usepackage{graphicx}

\usepackage{url}            
\usepackage{booktabs}       
\usepackage{array}
\usepackage{makecell}
\usepackage{multirow}

\newcolumntype{L}[1]{>{\raggedright\arraybackslash}p{#1}}

\usepackage{nicefrac}       
\usepackage{microtype}      
\usepackage{amsmath,latexsym,bm,mathtools,amssymb,amsthm}
\usepackage{amsfonts}
\usepackage{relsize} 
\usepackage{caption,subcaption}
\usepackage{wrapfig}

\usepackage[vlined,ruled,linesnumbered]{algorithm2e} 
\SetCommentSty{small}

\newcounter{subrtn}

\makeatletter
\newif\ifalg@subroutine
\alg@subroutinefalse

\AtBeginEnvironment{algorithm}{%
  \ifalg@subroutine\else
    \setcounter{subrtn}{0}%
  \fi
}

\newenvironment{subroutine}[1][]{%
  \alg@subroutinetrue
  \stepcounter{subrtn}%
  \edef\algparent{\thealgocf}
  \begin{algorithm}[#1]%
    \addtocounter{algocf}{-1}
    \renewcommand{\thealgocf}{\algparent.\arabic{subrtn}}
}{%
  \end{algorithm}%
  \alg@subroutinefalse
}
\makeatother

\usepackage{amsthm} 
\usepackage{thmtools} 
\usepackage{thm-restate}

\usepackage{xcolor}         
\definecolor{pku-red}{RGB}{139,0,18}
\usepackage[ colorlinks,citecolor=blue,linkcolor=pku-red,urlcolor=pku-red,bookmarks=true]{hyperref}
\usepackage[nameinlink]{cleveref}

\usepackage{enumitem}

\usepackage{libertine}

\usepackage{tikz}  

\usepackage{bbm}

\usepackage[hyperref = true, natbib = true, style = alphabetic, isbn=false, doi=false, maxbibnames=999]{biblatex}

\AtEveryBibitem{
\clearlist{address}
\clearfield{volume}
 \clearlist{location}
 \clearfield{month}
 \clearfield{series}
 \clearfield{eprint}
 \clearfield{number}
 
 \ifentrytype{book}{}{
  \clearname{editor}
 }
}

\usepackage{tcolorbox}

\usepackage{authblk}

\usepackage{subcaption}

\usepackage{acronym}

\newtheorem{theorem}{Theorem}[section]
\newtheorem{corollary}[theorem]{Corollary}

\newtheorem{lemma}[theorem]{Lemma}

\theoremstyle{definition}
\newtheorem{definition}[theorem]{Definition}

\newtheorem{example}[theorem]{Example}

\newtheorem{fact}[theorem]{Fact}
\crefname{fact}{fact}{facts}

\providecommand{\email}[1]{\href{mailto:#1}{\nolinkurl{#1}\xspace}}

\Crefname{equation}{Eq.}{Eqs.}

\newcommand{\calE}{\mathcal{E}}

\newcommand{\calI}{\mathcal{I}}

\newcommand{\calK}{\mathcal{K}}

\newcommand{\calT}{\mathcal{T}}

\newcommand{\calV}{\mathcal{V}}

\newcommand{\bbE}{\mathbb{E}}

\newcommand{\bbR}{\mathbb{R}}

\newcommand{\bbZ}{\mathbb{Z}}

\newcommand{\E}{\mathbb{E}}
\newcommand{\reg}{\mathrm{Reg}}

\newcommand{\simplex}{\Delta_K}

\newcommand{\KL}[2]{\mathrm{KL}(#1\|#2)}

\acrodef{FTL}[FTL]{follow--the--leader}
\acrodef{EWOO}[EWOO]{exponentially weighted online optimization}
\acrodef{BM}[BM]{Blum--Mansour}
\title{Calibeating Made Simple}
\author{Yurong Chen$^{1}$, Zhiyi Huang$^{2}$, Michael I. Jordan$^{1,3}$, Haipeng Luo$^{4}$\\
    $^1$Inria, École Normale Supérieure, PSL Research University \\
    $^2$The University of Hong Kong  \\
    $^3$University of California, Berkeley \\
    $^4$University of Southern California \\
    \email{yurong.chen@inria.fr},
    \email{zhiyi@cs.hku.hk},\\
    \email{jordan@cs.berkeley.edu},
    \email{haipengl@usc.edu}
}
\date{}

\begin{document}

\maketitle

\begin{abstract}
We study calibeating, the problem of post-processing external forecasts online to minimize cumulative losses and match an informativeness-based benchmark. 
Unlike prior work, which analyzed calibeating for specific losses with specific arguments, we reduce calibeating to existing online learning techniques and obtain results for general proper losses.
More concretely,
we first show that calibeating is minimax-equivalent to regret minimization. 
This recovers the $O(\log T)$ calibeating rate of \citet{foster2023calibeating} for the Brier and log losses and its optimality,
and yields new optimal calibeating rates for mixable losses and general bounded losses.
Second, we prove that multi-calibeating is minimax-equivalent to the combination of calibeating and the  classical expert problem.
This yields new optimal multi-calibeating rates for mixable losses, including Brier and log losses, and general bounded losses.
Finally, we obtain new bounds for achieving calibeating and calibration simultaneously for the Brier loss.
For binary predictions, our result gives the first calibrated algorithm that at the same time also achieves the optimal $O(\log T)$ calibeating rate.
\end{abstract}

\section{Introduction}

\label{sec:introduction}

Calibration has attracted growing attention in recent years as a desideratum for probabilistic prediction, motivated by the need to produce reliable probabilities for downstream decision-making~\citep{guo2017on}. 
Despite its appeal as a benchmark for reliability, 
however, 
calibration is not necessarily a meaningful test of forecasting expertise.
For example, online calibration can be achieved 
by randomized strategies 
without any knowledge of the data-generating process~\citep{foster1998asymptotic}.
Hence, calibration alone cannot distinguish true expertise from uninformative procedures.

To quantify and preserve forecasting expertise, 
\citet{foster2023calibeating} introduced \emph{calibeating} in a post-processing setting. In this setting, an external forecaster (e.g., a machine learning model) outputs a probabilistic forecast at each round,  
and then the learner produces its own forecast based on it. 
It is known that for proper losses such as the Brier and log losses, the cumulative score can be decomposed into a calibration term, which measures the reliability, and a refinement term, which measures the informativeness and skill. This motivates the question of whether one can improve reliability without sacrificing skill. 
Calibeating formalizes this goal by requiring the learner's  
loss to be as small as the external forecaster's refinement score (in other words, to ``beat''
the forecaster by its calibration error).

Existing work~\citep{foster2023calibeating,lee2022online} establishes online calibeating guarantees for 
the Brier and log losses,
and studies extensions 
such as 
beating multiple forecasters (multi-calibeating) and imposing simultaneous calibration constraints. These results rely on loss-specialized
analyses. More broadly, the fundamental statistical difficulty of calibeating and its relationship to standard online-learning problems have remained unclear, leaving open whether known bounds are optimal or how they 
generalize beyond the Brier and log losses. 

\subsection{Our Results}

We study calibeating from an online-learning perspective. Rather than analyzing different losses on a case-by-case basis, we identify simple reductions from (multi-)calibeating to standard online-learning primitives. 
This yields a ``plug-and-play'' analysis:
by instantiating the reductions with classical online-learning algorithms, we obtain general upper and lower bounds in a modular way.

\paragraph{Calibeating $=$ no-regret learning (\Cref{sec: calibeating}).}
We prove that calibeating is minimax-equivalent to regret minimization. 
\Cref{thm:upper bound} gives a reduction that turns any no-regret learner with regret bound $\alpha(T)$ into a calibeating algorithm with a corresponding bound that scales with $|Q|$, the number of distinct 
external forecast values over $T$ rounds.
The reduction exploits the fact that the refinement benchmark decomposes across distinct forecast values, allowing one to treat each corresponding subsequence independently. 
Instantiating this reduction recovers the $O(|Q|\log T)$ guarantees for the Brier and log losses from \citet{foster2023calibeating} and extends them to general mixable losses (\Cref{coro:upper-bound-mixable}).
We also obtain an $O(\sqrt{|Q|T})$ bound for general bounded proper losses (\Cref{coro:upper-bound-bounded-loss}).
Conversely, \Cref{thm:lower bound} provides a matching lower bound that completes the minimax-equivalence.

\paragraph{Multi-calibeating $=$ calibeating $+$ expert problem (\Cref{sec: multi-calibeating}).}
Next, we present (in \Cref{thm:upper-bound-multi-calibeating})  a simple decomposition of multi-calibeating into calibeating and the expert problem: 
run a separate calibeating subroutine for each forecaster to produce candidate predictions, then aggregate them with an expert algorithm.
The resulting multi-calibeating guarantee is the sum of the calibeating bound and the expert regret bound.
For mixable losses,
we obtain a logarithmic bound of 
$O(\log N +  |Q^{(n)}|\log T)$ (\Cref{coro:upper-bound-multi-mixable}), where $N$ is the number of forecasters and $Q^{(n)}$ is the set of distinct forecasts produced by forecaster $n$.
This improves exponentially over the polynomial dependence on $N$ in \citet{foster2023calibeating} and the polynomial dependence on $T$ in \citet{lee2022online}. 
We complement this with a lower-bound reduction (\Cref{thm:lower bound-multi-calibeating}),  showing that multi-calibeating inherits hardness from both the expert problem and the calibeating problem. 
This yields matching lower bounds for Brier and log losses and shows the tightness of our results (\Cref{coro:lower-bound-multi-brier-log}). 

\paragraph{Simultaneous (multi-)calibeating and calibration for Brier loss (\Cref{sec:simultaneous-calibration-calibeating}).}
Finally, we provide new bounds for achieving calibeating and calibration simultaneously. 
We propose a meta-algorithm that tracks an arbitrary reference 
algorithm
while ensuring calibration (\Cref{thm:calibeating-calibration}).
The construction employs two existing online learning primitives:
the reduction by \cite{blum2007from} to enforce calibration via the calibration-swap-regret connection, and a two-expert algorithm by \citet{sani2014exploiting} to aggregate the predictions from the \ac{BM} reduction and the reference algorithm. 
Instantiating the reference with the (multi-)calibeating algorithms from the previous sections, for the Brier loss, we obtain for the binary case  the optimal logarithmic (multi-)calibeating rate of $O(\log N + |Q^{(n)}|\log T)$
while ensuring a sublinear $\ell_2$-calibration error of 
order 
$\tilde{O}(\sqrt{T})$ (\Cref{coro:binary-calibeating-calibration}).
This improves the polynomial $T$-dependence on the calibeating side 
in \citet{foster2023calibeating} and improves both sides compared to \citet{lee2022online}. 
For multi-class outcomes, we derive explicit tradeoffs between (multi-)calibeating and calibration (\Cref{coro:multi-class-calibeating-calibration}). In particular, at one extreme, we recover the known dependence on $T$ for calibration~\citep{foster2023calibeating,fishelson2025full} while dropping the $|Q^{(n)}|$ dependence in \citet{foster2023calibeating}.
For a summary of our results and comparisons with prior work, see \Cref{tab:main}.

\newlength{\settingwidth}
\settowidth{\settingwidth}{Simultaneous}
\newlength{\lossclass}
\settowidth{\lossclass}{Brier (multi-class)}
\newlength{\ours}
\settowidth{\ours}{$(O(\log T),\tilde O(\sqrt{T}))$\Cref{coro:binary-calibeating-calibration}}

\begin{table}[t]
\centering
\small
\caption{Comparison of prior and our guarantees  in $N$, $T$, $K$ (the number of outcomes), and $|Q|$ (we assume $|Q^{(n)}| = |Q|$ for simplicity). For simultaneous calibeating and calibration, the first rate is for calibeating and the second for calibration. We omit polynomial dependence on $K$ for presentation clarity, and $\tilde{O}$ omits logarithmic dependence on $T$. 
The simultaneous results of \citet{foster2023calibeating} are only for calibeating (but not multi-calibeating), so we only show results for calibeating for comparison. 
The results of \citet{lee2022online} are only for binary outcomes. }
\label{tab:main}
\resizebox{\textwidth}{!}{
\begin{tabular}{L{1.76cm} L{1.46cm} L{5.2cm} L{3.8cm}L{1.81cm}}
\toprule
Setting & Loss class & Prior work & \textbf{This paper} \\
\midrule
\multirow{4}{*}{Calibeating} & Mixable & -- & $\Theta(|Q|\log T)$ & (Cor.~\ref{coro:upper-bound-mixable}, \ref{coro:lower-bound-brier-log}) \\
							& ~ - Brier & $\Theta(|Q|\log T)$~\footnotemark[1]  & \\
							& ~ - Log & $O(|Q|\log T)$~\footnotemark[1] & \\[.5ex]
							& Bounded & -- & $\Theta(\sqrt{|Q|KT})$ & (Cor.~\ref{coro:upper-bound-bounded-loss}, \ref{coro:lower-bound-bounded-loss})\\
\midrule
\multirow{5}{*}{\makecell{Multi-\\Calibeating}} & Mixable & -- & $\Theta(\log N + |Q|\log T)$ & (Cor.~\ref{coro:upper-bound-multi-mixable}, \ref{coro:lower-bound-multi-brier-log}) \\
							& \multirow{3}{*}{~ - Brier} 	& $O((N+|Q|)\log T)$~\footnotemark[1] $O(\sqrt{NT}+|Q|\log T)$~\footnotemark[1] $\tilde{O}(\sqrt{|Q|}(\log N)^{\frac{1}{4}}T^{\frac{3}{4}} ) $~\footnotemark[2] & \\[2ex]
& Bounded & -- & $\Theta(\sqrt{T\log N} + \sqrt{|Q|KT})$ & (Cor.~\ref{coro:upper-bound-multi-bounded-loss}, \ref{coro:lower-bound-multi-bounded})\\
\midrule
\multirow{3}{*}{\makecell{Calibeating \\\& Calibration}}
& Brier & & \\
& ~ - binary
& $\tilde{O}(|Q|^{\frac{2}{3}}T^{\frac{1}{3}}), \tilde{O}(|Q|^{\frac{2}{3}}T^{\frac{1}{3}})$~\footnotemark[1] & $O(|Q|\log T), \tilde O(\sqrt{T})$ & (Cor.~\ref{coro:binary-calibeating-calibration})\\
&  ~ - $K$-class & $\tilde{O}(|Q|^{\frac{2}{K+1}}T^{\frac{K-1}{K+1}}), \tilde{O}(|Q|^{\frac{2}{K+1}}T^{\frac{K-1}{K+1}})$~\footnotemark[1] & $\tilde{O}  ( |Q|+  T^{\frac{K-1}{K+1}}), \tilde{O} (T^{\frac{K-1}{K+1}})$ & (Cor.~\ref{coro:multi-class-calibeating-calibration})  \\
\bottomrule 
\multicolumn{4}{l}{\footnotemark[1] \citealt{foster2023calibeating},\;\;\footnotemark[2] \citealt{lee2022online}}
\end{tabular}
}
\end{table}

\subsection{Related Work}

\paragraph{Calibeating.}In the seminal work proposing calibeating, \citet{foster2023calibeating} give online guarantees for the Brier and log losses via a bin-wise estimation viewpoint. They also study extensions to multiple forecasters and to simultaneous calibeating and calibration. 
\citet{lee2022online} formulate simultaneous multi-calibration and multi-calibeating as an online multi-objective optimization problem, achieving favorable dependence on the number of external forecasters but suboptimal dependence on the time horizon.  
In comparison, our results are obtained via reductions that connect calibeating to standard online-learning problems. 
Finally, as an application, \citet{gupta2023online} apply calibeating as a robustness layer on top of online Platt scaling to guarantee adversarial calibration in binary classification while preserving predictive performance.

\paragraph{Online recalibration.} Online recalibration is studied in the same post-processing setting as calibeating, but it benchmarks performance by proper-loss regret rather than refinement.
The goal is to achieve small regret relative to the external forecaster while simultaneously ensuring calibrated predictions~\citep{marx2025calibrated,deshpande2024calibrated}.  
For binary classification, \citet{kuleshov2017estimating} provide adversarial online guarantees, and \citet{okoroafor2024faster} obtain improved bounds and explicit regret versus $\ell_1$-calibration tradeoffs via Blackwell approachability for 
strictly proper losses. 
These tradeoffs yield sublinear but typically polynomial-in-$T$ rates. While we focus on the Brier loss in our simultaneous guarantee, we target the stronger refinement benchmark and achieve logarithmic-in-$T$ rates while still ensuring sublinear $\ell_2$-calibration. 

\paragraph{Calibration and proper scoring loss.}Proper scoring losses admit classical decompositions into a reliability (calibration) term and an informativeness (refinement) term~\citep{dawid,sanders1963on,brocker2009reliability}. In this spirit, $\ell_2$-calibration~\citep{foster1998asymptotic} and KL-calibration~\citep{luo2025simultaneous} can be viewed as online analogues of the calibration term for the Brier loss and the log loss, respectively; more generally, this motivates defining online calibration measures compatible with arbitrary proper scoring losses.
Several calibration notions, including $\ell_2$-calibration~\citep{fishelson2025full} and KL-calibration~\citep{luo2025simultaneous}, have been shown to be equivalent to swap-regret objectives.
We exploit this connection in our simultaneous guarantees by enforcing (pseudo-)swap regret via the \ac{BM} reduction~\citep{blum2007from}.

\section{Model}
\label{sec:preliminaries}

We consider an online prediction problem over a finite outcome space with $N$ external forecasts.  
Let $K\geq 2$ be the number of possible outcomes, and $\simplex \coloneqq \{p\in\bbR^K_{\geq 0}: \sum^K_{k=1} p_k = 1\}$ be the probability simplex. 
We let $[n]$ denote the set $\{1,\dots, n\}$ for any positive integer $n$. 
The outcome space is denoted by $\calE \coloneqq \{e_i: i \in [K] \} \subseteq \simplex$, 
 where $e_i$ is the $i$-th standard basis vector.

The interaction proceeds for $T$ rounds. At each round $t \in [T]$, the learner first observes $N$ external forecasts, $q^{(n)}_t\in \simplex$, $n \in [N]$, and  makes its own prediction $p_t \in \simplex$. The outcome $y_t\in\calE$ is then revealed, and the learner incurs loss $\ell(p_t,y_t)$. 
For simplicity, we assume that $q_{1:T} \coloneqq (q_t)^T_{t=1}$ and $y_{1:T} \coloneqq (y_{t})^T_{t=1}$ are generated by an oblivious adversary, i.e., they are decided at time $t=0$ with complete knowledge of the learner's algorithm (but not its random bits). 

Throughout, we consider a proper scoring loss
$\ell: \Delta_K \times \calE \rightarrow \bbR$, i.e., losses such that
for any $q\in \Delta_K$, $q\in \arg\min_{p \in \Delta_K} \E_{y\sim q}[\ell(p,y)]$. 
We write $\ell(p,q)\coloneqq \E_{y\sim q}[\ell(p,y)]$.
Let $\mathbf{1}\{\cdot\}$ denote the indicator function, which equals one if the condition holds and zero otherwise. 
Given a prediction sequence $p_{1:T}$ and outcome sequence $y_{1:T}$, for any $p\in \simplex$, denote the number of times the learner predicts $p$ as $n_T(p) \coloneqq \sum^T_{t=1}	\mathbf{1}\{p_t = p\}$,
and the empirical outcome distribution conditioned on prediction $p$ as $\rho^p_{T}(y) \coloneqq \frac{1}{n_T(p)}\sum^T_{t=1} \mathbf{1}\{p_t = p, y_t = y\}$ for $y\in\calE$, whenever $n_T(p)>0$. 
With these definitions, the cumulative loss, refinement score, and calibration error are defined as follows.

\begin{definition}
\label{lmm:operational-form}
The \emph{cumulative loss} of predictions $p_{1:T}$ under outcomes $y_{1:T}$ is
\begin{align*}
L_T(p_{1:T},y_{1:T}) \coloneqq \sum^T_{t=1}\ell(p_t,y_t)
~.
\end{align*}
The \emph{refinement score} is
\begin{align*}
R_T(p_{1:T},y_{1:T}) \coloneqq \sum_{p} n_T(p)\ell(\rho^p_T, \rho^p_T) = \sum_{p}\min_{q \in \simplex}\sum_{t:p_t=p}\ell(q, y_t)
~.
\end{align*}
Finally, the \emph{calibration error} is
\begin{align*}
K_T(p_{1:T}, y_{1:T})\coloneqq L_T(p_{1:T},y_{1:T}) - R_T(p_{1:T},y_{1:T})
~.
\end{align*}
\end{definition}

By construction, $L_T = R_T + K_T$ and $K_T\ge 0$. 
Moreover, 
$K_T$ coincides with the full-swap-regret notion of \citet{fishelson2025full}, while 
$R_T$ corresponds to the best-in-hindsight swap-regret benchmark. Indeed, for each prediction $p$, the refinement term equals the loss of the best constant predictor over rounds with $p_t=p$. Thus, the refinement score measures the informativeness of the forecasts: sequences that induce finer bins with lower within-bin variability achieve smaller refinement.  In contrast, the calibration error measures within-bin reliability, i.e., how close the issued prediction $p$ is to the empirical conditional distribution $\rho^p_T$ on the corresponding subsequence.

Proper scoring losses admit a classic decomposition into terms measuring the
informativeness (or refinement) of forecasts and their reliability (or calibration)
in a probabilistic setting; see, e.g., \citet{brocker2009reliability, dawid}.
\Cref{lmm:operational-form} can be seen as the empirical counterparts of these quantities. 

\begin{example} For \emph{Brier loss} $\ell(p,y)= \|p-y\|^2_2$, 
the refinement score 	equals the weighted sum of within-bin variances, and the calibration error becomes the
    $\ell_2$-calibration~\citep{foster1998asymptotic},
	\begin{align*}
	R_T(p_{1:T}, y_{1:T}) = \sum_{p} n_T(p) \sum_{t:p_t =p} \frac{1}{n_T(p)} \|\rho^p_T - y_t\|^2_2 ~, \quad K_T(p_{1:T},y_{1:T}) = \sum_{p}{n_T(p)}\|p-\rho^p_T\|^2_2 ~.		
	\end{align*}
\end{example}

\begin{example}
Denote the Shannon entropy under distribution $p$ to be $H(p) = -\sum_{k}p_k \log p_k$. For \emph{log loss} $\ell(p,y) = -\sum^K_{k=1}y_k\log p_k$, 
the refinement score 
equals the weighted sum of the Shannon entropy within each bin, and the calibration error becomes the 
KL--calibration \citep{luo2025simultaneous},
	\begin{align*}	
	R_T(p_{1:T}, y_{1:T}) =\sum_{p} n_T(p) H(\rho^p_T) ~,\quad K_T(p_{1:T},y_{1:T}) = \sum_{p}{n_T(p)} \, \KL{\rho^p_T}{p} ~.	
	\end{align*}
\end{example}  

Motivated by this decomposition, \citet{foster2023calibeating} compare the learner to the external forecaster's refinement score and define the notions of \emph{calibeating} and \emph{multi-calibeating}. 

\begin{definition}[Calibeating and Multi-Calibeating]
A learner 
is $\alpha(T)$-{\emph {multi-calibeating}} w.r.t.\ loss $\ell$ if for any
external forecasts $\{q^{(n)}_{1:T}\}^N_{n=1}$ and outcomes 
$y_{1:T}$, the learner's predictions $p_{1:T}$ satisfy
\begin{align}
\label{eq:alternative-calibeating}
L_T(p_{1:T},y_{1:T}) \leq  R_T(q^{(n)}_{1:T}, y_{1:T}) + \alpha(T)
~,\forall n \in [N]
~.
\end{align}

We call $\alpha(T)$ the multi-calibeating rate.
We say the learner is multi-calibeating if $\alpha(T)=o(T)$. 
When \eqref{eq:alternative-calibeating} holds in expectation over the learner's randomness, we call $\alpha(T)$ the expected multi-calibeating rate. 
When there is only $N=1$ external forecast, we simply say calibeating. 
\end{definition}

We also introduce another performance measure called calibration. 

\begin{definition}[Calibration]
\label{def:calibration}
A learner is $\beta(T)$-calibrated w.r.t.\ loss $\ell$ if for any outcome sequences $y_{1:T}$ (and any external forecasts),
the learner's predictions $p_{1:T}$ satisfy
\begin{align}
\label{eq:def-calibration}
K_T(p_{1:T}, y_{1:T}) \leq \beta(T). 	
\end{align}
We call $\beta(T)$  the calibration rate and say the algorithm is calibrated if $\beta(T)=o(T)$. 
When \eqref{eq:def-calibration} holds in expectation over the algorithm's randomness, we call $\beta(T)$  the expected calibration rate.  
\end{definition}

Note that calibeating and calibration are incomparable in general.
Calibeating only guarantees that the learner’s loss is no larger than the forecaster’s loss minus the forecaster’s calibration error, i.e., it competes with the \emph{external forecaster}’s refinement score. 
The learner's calibration error might not vanish if
it itself attains a low refinement term.

\section{Calibeating $=$ No-Regret Learning}
\label{sec: calibeating}

This section considers the calibeating problem, i.e., when there is only one external forecast every round. 
Let $Q\coloneqq \{q_t: t\in [T]\}$ denote the set of distinct external forecast values that appear over the horizon.%
\footnote{We also use $Q$ to denote the set of possible external forecast values for lower bound results.} 
\citet{foster2023calibeating} study the Brier and log losses and give algorithms 
with calibeating rate of $O(|Q|\log T)$. 
We recover and extend their results by reductions to no-regret learning. 

First, define the regret of predictions $p_{1:T}$ under outcomes $y_{1:T}$ to be
\begin{align*}
\reg_T(p_{1:T},y_{1:T})
\coloneqq
\sum_{t=1}^T \ell(p_t,y_t)
-
\min_{p\in\simplex}\sum_{t=1}^T \ell(p,y_t)~.
\end{align*}
We say an algorithm has (expected-)regret of $\alpha(T)$ if $\reg_T (p_{1:T},y_{1:T}) \leq \alpha(T)$ always holds (in expectation). 
The following theorem shows that calibeating reduces to no-regret learning. 

\begin{theorem}
\label{thm:upper bound}
For any proper loss $\ell$ and any online algorithm $\mathsf{A}$ with regret $\alpha(T)$,
where $\alpha$ is a concave function, \Cref{alg:calibeating-from-regret} is $|Q|\alpha(T/|Q|)$-calibeating.
\end{theorem}

\begin{proof}
The reduction partitions the rounds $t \in [T]$ by the external forecast value $q_t$, and runs an independent copy of the no-regret learner $\mathsf{A}$ for each forecast value. 
Formally, 
for any external forecast $q$ that appeared at least once, run a separate copy of $\mathsf{A}$, denoted as $\mathsf{A}_q$, on the subset of rounds $\calI_q\coloneq \{t: q_t = q\}$. 
For each subsequence, we have
\begin{align*}
    \sum_{t: q_t=q}\ell(p_t,y_t) - \min_{p\in\simplex}\sum_{t: q_t=q}\ell(p,y_t) \leq \alpha(n_T(q))~.
\end{align*}

Summing up over all the subsequences and by \Cref{lmm:operational-form}, we have 
\begin{align*}
L_T(p_{1:T},y_{1:T}) - R_T(q_{1:T},y_{1:T}) 
& \leq \sum_{q}\alpha(n_T(q)) \\
& \leq |Q| \alpha \left(\frac{\sum_{q}n_T(q)}{|Q|}\right) \tag{Jensen's inequality} \\
& = |Q| \alpha \left(\frac{T}{|Q|}\right)~,
\end{align*}
which finishes the proof.
\end{proof}
\begin{algorithm}
\DontPrintSemicolon
\caption{Calibeating by Bin-Wise No-Regret}
\label{alg:calibeating-from-regret}
\KwIn{Online learner $\mathsf{A}$.}
\For{$t = 1$ \KwTo $T$}{
  \tcp{prediction}
    Observe external forecasts $q_t\in\Delta_K$.\\
    \If{$\mathsf{A}_{q_t}$ is uninitialized}{
        Initialize a fresh copy $\mathsf{A}_{q_t} \leftarrow \mathsf{A}$. 
    }
    Query $\mathsf{A}_{q_t}$ and obtain prediction $p_t \in \simplex$.\\[0.5ex]
    \tcp{update}
    Observe outcome $y_t$ and incur loss $\ell(p_t,y_t)$.\\
    Update $\mathsf{A}_{q_t}$ with $y_t$.
}
\end{algorithm}

We note that common regret bounds obtained for standard online algorithms are all concave in $T$, e.g., they are often of the form $O((\log T)^\alpha T^{\beta})$ for some $\alpha>0$ and $\beta \in [0,1)$. The algorithm of \citet{foster2023calibeating} can be recovered as a special case of \Cref{thm:upper bound}, with the online algorithm being \ac{FTL} (with a standard interior restriction for the log loss to avoid the unbounded boundary).
Moreover, \Cref{thm:upper bound} readily obtains $O(\sqrt{|Q|KT})$ for bounded proper losses~\citep{luo2024optimal} and $O(|Q|\log T)$ for mixable losses~\citep{hazan2007logarithmic}, which encompasses Brier and log losses as special cases.

\begin{corollary}
\label{coro:upper-bound-bounded-loss}
For bounded proper loss $\ell$, instantiating the online learner in \Cref{thm:upper bound} with the follow-the-perturbed-leader algorithm of \citet{luo2024optimal}  yields an algorithm with expected calibeating rate of $O(\sqrt{|Q|KT})$.
\end{corollary}

We note that, inherited from the no-regret guarantees in \citet{luo2024optimal}, the algorithm in \Cref{coro:upper-bound-bounded-loss} can actually achieve $O(\sqrt{|Q|KT})$ \emph{simultaneously} for all bounded proper losses. 

\begin{definition}
\label{def:exp-concavity}
A convex function $\ell(\cdot)$ is $\eta$-mixable if  for any probability distribution $\pi \in \Delta(\Delta_K)$, there exists a prediction $p_\pi \in \Delta_K$ such that
$e^{-\eta \ell(p_\pi, y)} \geq \int e^{-\eta \ell(p,y)} \pi(\mathrm{d} p)$ holds for all $y \in \calE$.
\end{definition}

\begin{corollary}
\label{coro:upper-bound-mixable}
For an $\eta$-mixable loss  $\ell$ (e.g., Brier and log losses), instantiating the online learner in \Cref{thm:upper bound} with \ac{EWOO}~\citep{hazan2007logarithmic}  yields an algorithm with expected calibeating rate of  $O(|Q| \log T)$.
\end{corollary}

Besides the upper bound, we also prove that any lower bound for no-regret learning with a proper loss implies a lower bound for calibeating.
Combining with \Cref{thm:upper bound}, our results show that calibeating is minimax-equivalent to regret minimization.  We defer the proof to \Cref{app:lower-bound}. 

\begin{restatable}{theorem}{thmlowerbound}
\label{thm:lower bound}
For any proper loss  $\ell$, denote the optimal regret bound as
\begin{align}
\label{eq:lower-bound-regret}
\beta(T) \coloneqq \inf_{\mathsf{A}} \sup _{y_{1:T}\in \calE^T} \mathbb{E}_{p_{1:T}\sim \mathsf{A}}\left[\sum_{t=1}^T \ell\left(p_t, y_t\right)-\min _{p \in \simplex} \sum_{t=1}^T \ell\left(p, y_t\right) \right]
~,
\end{align}
where $\mathsf{A}$ ranges over (possibly randomized) online algorithms. 
Then, every algorithm is at best $|Q| \beta(\lfloor T / |Q|\rfloor)$-calibeating.

\end{restatable}

Combining \Cref{thm:lower bound} with known regret lower bounds for bounded proper losses~\citep{luo2024optimal} and Brier and log losses~\citep{cesabianchi2006prediction} yields the following. 

\begin{corollary}
\label{coro:lower-bound-bounded-loss}
There exist bounded proper losses 
with calibeating rate
at least $\Omega(\sqrt{|Q|KT})$.
\end{corollary}

\begin{corollary}
\label{coro:lower-bound-brier-log}
For the Brier and log losses, the calibeating rate is at least $\Omega (|Q|\log (T/|Q|))$. 
\end{corollary}

\section{Multi-Calibeating $=$ Calibeating $+$ Expert Problem}

\label{sec: multi-calibeating}

Next, we consider the multi-calibeating problem. 
\citet{foster2023calibeating} obtain multi-calibeating rates of $O((N+|Q|)\log T)$ and  $O(\sqrt{NT}+|Q|\log T)$, via Blackwell approachability and online linear regression.
\citet{lee2022online} achieve a rate logarithmic in $N$, but polynomial in $T$ (more precisely, $\smash{\tilde{O}(\sqrt{|Q|}(\log N)^{\frac{1}{4}}T^{\frac{3}{4}} )} $ for the optimal choice of parameters).

This section presents a simple reduction from multi-calibeating to the expert problem.
Via that reduction, we achieve the optimal multi-calibeating rates. 
 
\paragraph{Expert Problem.} The interaction protocol in this problem is the same as in multi-calibeating: at each round $t$, the learner observes $N$ expert predictions $\smash{\{p_t^{(n)}\}_{n}\subseteq\Delta_K}$, and makes its own prediction $p_t \in \simplex$.
An experts algorithm $\mathsf{E}$ 
achieves regret $\gamma(T)$ if for every sequence
$\smash{\{(p_t^{(1:N)},y_t)\}_{t=1}^T}$,
\begin{align}
	\label{eq:experts-regret-pred}
\mathbb E \left[\sum_{t=1}^T \ell(p_t,y_t)\right]
\le
\min_{n\in[N]}\sum_{t=1}^T \ell(p_t^{(n)},y_t) + \gamma(T),
\end{align}
where the expectation is over the randomness of $\mathsf{E}$. 

\medskip

Comparing this definition of regret with the definition of multi-calibeating rate, the only difference is that the latter remaps the experts/forecasters' predictions optimally, while the former does not. 
The remapping of each individual forecaster is precisely the problem of calibeating.
Hence, we run a separate calibeating algorithm for each forecaster, and use an experts algorithm to aggregate their decision.
See \Cref{alg:bestN-general} for a formal description.

\begin{algorithm}
\DontPrintSemicolon
\caption{Multicalibeating by Expert Aggregation}
\label{alg:bestN-general}
\textbf{Sub-routines:}\\[.5ex]
\begin{itemize}[after=\vspace{-0.5\baselineskip},parsep=0pt]
\setlength\itemsep{0pt}
\item For each forecaster $n \in [N]$, a separate calibeating algorithm $\mathsf{A}^{(n)}$ (\Cref{alg:calibeating-from-regret}).
\item Experts algorithm $\mathsf{E}$~\citep[e.g., Hedge,][]{freund1997decision}.
\end{itemize}
\For{$t = 1$ \KwTo $T$}{
  \tcp{prediction}
    Observe external forecasts $q_t^{(1)},\dots,q_t^{(N)}\in\Delta_K$.\\
    For each $n \in [N]$, query $\mathsf{A}^{(n)}$ with forecast $q_t^{(n)}$ to get its prediction $p_t^{(n)}\in\Delta_K$.\\
    Query $\mathsf{E}$ with $\{p_t^{(n)}\}_{n=1}^N$ as the experts' forecasts, and follow its prediction $p_t$. \\[0.5ex]
    \tcp{update}
    Observe outcome $y_t$ and update $\mathsf{A}^{(n)}$ for each $n \in [N]$ with this outcome.\\
    Update $\mathsf{E}$ with  $\ell(p_t^{(n)},y_t)$ as the loss of expert $n \in [N]$.
}
\end{algorithm}

\begin{theorem}
\label{thm:upper-bound-multi-calibeating}
For any loss function $\ell$, any calibeating algorithm with rate $\alpha(T)$, and any experts algorithm with regret $\gamma(T)$, Algorithm \ref{alg:bestN-general} is 
$(\alpha(T)+\gamma(T))$-calibeating.
\end{theorem}

\begin{proof}
By the regret bound of algorithm $\mathsf{E}$, for any $n \in [N]$, we have
\begin{align*}
\E\, [L_T(p_{1:T},y_{1:T})]  
\leq  
L_T(p^{(n)}_{1:T},y_{1:T}) + \gamma(T)
~. 
\end{align*}
By the calibeating rate of algorithm $\mathsf{A}$, we have
\begin{align*}
L_T(p^{(n)}_{1:T},y_{1:T}) \leq R_T(q^{(n)}_{1:T},y_{1:T}) +\alpha(T).
\end{align*}
Combining these inequalities yields a multi-calibeating rate of $\alpha(T)+ \gamma(T)$. 
\end{proof}

Let $\smash{Q^{(n)} \coloneqq \{q^{(n)}_t: t \in [T]\}}$ denote the set of distinct external forecasts made by forecaster $n \in [N]$. 
We assume $\smash{|Q^{(n)}| = |Q|}$ for all $n$ for simplicity.\footnote{Our results also hold when $|Q^{(n)}|$s differ across forecasters, and the resulting bounds adapt to specific forecasters. } 
By the regret bounds of Hedge (e.g.,~\citealp[Theorems~2.1 and~2.2]{bubeck2011introduction}), and the calibeating rates in \Cref{coro:upper-bound-bounded-loss,coro:upper-bound-mixable}, we get the following corollaries. In contrast to the loss-oblivious property of \Cref{coro:upper-bound-bounded-loss}, the algorithm in \Cref{coro:upper-bound-multi-bounded-loss} requires a fixed $\ell$, as the experts algorithm relies on the loss values to update. 
\begin{corollary}
\label{coro:upper-bound-multi-bounded-loss}
For any bounded proper loss $\ell$, there exists an
algorithm with an expected multi-calibeating rate of 
$O(\sqrt{T\log N}+\sqrt{|Q| K T})$. 
\end{corollary}

\begin{corollary}
\label{coro:upper-bound-multi-mixable}
For any $\eta$-mixable loss  $\ell$ (e.g., Brier and Log losses), there exists an algorithm with expected multi-calibeating rate of
$O(\log N + |Q|\log T)$.
\end{corollary}
 
We show a minimax-equivalence of the two problems. 
Since now external forecasts are involved, we consider calibeating rates and expert regrets as functions of both the time round and the number of possible distinct external forecast values. Similarly to $Q^{(n)}$, given an instance of the expert problem, let $\smash{P^{(n)}\coloneqq \{p^{(n)}_t: t \in [T]\}}$ denote the set of possible expert predictions made by expert forecaster $\smash{n\in [N]}$. We have the following theorem. 

\begin{restatable}{theorem}{thmlowerboundmulticalibeating}
\label{thm:lower bound-multi-calibeating}
For any proper loss $\ell$,
suppose there exist functions $\phi, \lambda: \bbZ^2 \rightarrow \bbR$ such that for any $T$ and $m$,
\begin{align*}
\inf_{\mathsf{A}} \sup_{\substack{\left(q_{1:T}, y_{1:T}\right):\\\forall n,|Q^{(n)}|\leq m}} 
\E_{p_{1:T}\sim \mathsf{A}}\left[L_T(p_{1:T},y_{1:T}) - R_T(q_{1:T},y_{1:T})\right] 
\geq  \phi(T,m),
\end{align*}
where $\mathsf{A}$ ranges over all randomized calibeating algorithms, and,
\begin{align*}
\inf_{\mathsf{E}} \sup_{\substack{(p^{(1:N)}_{1:T}, y_{1:T}):\\\forall n, |P^{(n)}|\leq m}} 
\E_{p_{1:T}\sim \mathsf{E}}\left[\sum^T_{t=1}\ell(p_t,y_t) - \min_{n\in[N]}\sum^T_{t=1} \ell(p^{(n)}_t, y_t)\right] 
\geq  \lambda(T,m),
\end{align*}
where $\mathsf{E}$ ranges over all randomized expert algorithms. 
Then,
\begin{align*}
\inf_{\mathsf{M}} \sup_{\substack{(q_{1:T}^{(1:N)},y_{1:T}):\\\forall n, |Q^{(n)}|\leq m}}
\E_{p_{1:T}\sim \mathsf{M}}\left[
L_T(p_{1:T},y_{1:T}) - \min_{n\in[N]}R_T(q^{(n)}_{1:T},y_{1:T})
\right] 
\geq \max\left\{\phi(T,m), \lambda(T,m)\right\},
\end{align*}
where $\mathsf{M}$ ranges over all multi-calibeating algorithms. 
\end{restatable}

By 
known lower bounds for expert problems
 when the expert predictions can be arbitrary values~\citep{cesabianchi2006prediction}, and that the lower-bound examples can be obtained when the size of distinct expert prediction values is constant, we show the following lower bounds for multi-calibeating, 
matching the upper bounds.

\begin{corollary}
\label{coro:lower-bound-multi-bounded}
There exist bounded proper losses under which the multi-calibeating rates 
are at least
$\Omega(\sqrt{T\log N}+\sqrt{|Q|KT})$. 
\end{corollary}

\begin{corollary}
\label{coro:lower-bound-multi-brier-log}
For Brier and log losses, the multi-calibeating rate of any algorithm is 
at least 
$\Omega(\log N+|Q|\log (T/|Q|))$.
\end{corollary}

\section{Calibeating and Calibration at the Same Time}
\label{sec:simultaneous-calibration-calibeating}

In this section, we consider the problem of achieving simultaneous calibeating and calibration.
Existing approaches focus on the Brier loss. \citet{foster2023calibeating} obtain simultaneous rates of $\smash{\tilde{O}(|Q|^{\frac{2}{K+1}}T^{\frac{K-1}{K+1}}))}$ via bin refinement and stochastic fixed-point methods, while \citet{lee2022online} obtain $\smash{\tilde{O}(\sqrt{|Q|}(\log N)^{\frac{1}{4}}T^{\frac{3}{4}} ) }$ in the binary case after parameter tuning.

We focus on the Brier loss and provide new and improved bounds for simultaneous calibeating and calibration. Specifically, we provide a meta-algorithm (Algorithm \ref{alg:calibeating-calibration}) which, for any given external reference algorithm $\mathsf{A}^*$, keeps careful track of the losses of $\mathsf{A}^*$ while ensuring calibration. 

\begin{theorem}
\label{thm:calibeating-calibration}
For Brier loss,  any $\varepsilon \in (0, 1)$, and any reference algorithm $\mathsf{A}^*$, \Cref{alg:calibeating-calibration} simultaneously guarantees an expected regret of at most $O(\varepsilon^2 T)$ compared to $\mathsf{A}^*$, and a calibration rate of at most $O_{K, \log T}(\sqrt{T} + \frac{1}{\varepsilon^{K-1}} \log \frac{1}{\varepsilon} + \varepsilon^2 T)$ with high probability.
\end{theorem}
Here, the $O_{K, \log T}$ notation hides a factor polynomial in $K$ and $\log T$ for readability.
We hide the $\log T$ factors because for the calibration error, the dominant dependence in $T$ is polynomial,
and we hide the $K$ factors because the bounds degenerate to the trivial $O(T)$ when $K$ gets larger and larger.
The bounds from previous works are also polynomial in $K$.

Let $\mathsf{A}^*$ be a multi-calibeating algorithm from the previous sections.
With $\varepsilon = \sqrt{\frac{\log T}{T}}$, for  $K = 2$, 
we obtain the optimal calibeating rate, improving the polynomial-in-$T$ dependence in \citet{foster2023calibeating}.

\begin{corollary}
    \label{coro:binary-calibeating-calibration}
	For Brier loss with binary outcomes, there is an algorithm with an expected multi-calibeating rate of at most $O(\log N +|Q|\log T)$, and a calibration rate of at most $O_{K,\log T}(\sqrt{T })$ with high probability.
\end{corollary}

With $\varepsilon = (\frac{\log T}{T})^{\frac{1}{K+1}}$ for $K \ge 3$, we achieve the same calibration rate as in \citet{fishelson2025full,foster2023calibeating}, and drop the $|Q|-$dependence in \citet{foster2023calibeating}.

\begin{corollary}
\label{coro:multi-class-calibeating-calibration}
	For Brier loss and $K \ge 3$ outcomes, there is an algorithm with an expected multi-calibeating rate of at most $O_{K, \log T} (\log N + |Q|+ T^\frac{K-1}{K+1})$, and a calibration rate of at most $O_{K, \log T}(T^\frac{K-1}{K+1})$ with high probability. 
\end{corollary}

For $K \ge 3$ outcomes, we can also lower the multi-calibeating rate at the cost of raising the calibration rate, by choosing a different $\varepsilon$.

\begin{restatable}{corollary}{coromulticlasscalibeatingcalibrationtradeoff}
    \label{coro:multi-class-calibeating-calibration-tradeoff}
	For Brier loss and  $K \ge 3$ outcomes, for any $x \in (\frac{K-3}{K-1}, \frac{K-1}{K+1}]$, there is an algorithm with expected multi-calibeating rate of at most $O_{K, \log T}  (\log N +|Q| + T^x)$, and a calibration rate of at most $\smash{O_{K, \log T}(T^{\frac{(K-1)(1-x)}{2}})}$ with high probability. 
\end{restatable}

\subsection{Algorithm}

\paragraph{Discretization and rounding.}

To achieve calibration, it is necessary to focus on a finite set of predictions via discretization.
For that, we consider a triangulation of $\simplex$ and randomly round each prediction to a vertex of the triangulation (recall that $\ell$ is fixed to the Brier loss in this section).

\begin{lemma}[\citealt{fishelson2025full}]
	\label{lem:discretization}
	For any $\varepsilon \in (0, 1)$, there is a subset of predictions $\calK^\varepsilon \subset \simplex$ of size $M = |\calK^\varepsilon| = O(\sqrt{K} \,\varepsilon^{-K+1})$, and a rounding scheme $\mathsf{H} : \simplex \rightarrow \Delta\left(\calK^\varepsilon\right)$ that maps an arbitrary prediction $q \in \simplex$ to a distribution over those in $\calK^\varepsilon$, such that for any outcome $y \in \calE$, we have $
		\bbE_{s \sim \mathsf{H}(q)} \, [\ell(s, y)]  \le \ell(q, y) + O(\varepsilon^2)$.
\end{lemma}

\paragraph{Blum-Mansour reduction.}
For the connection between calibration and no-swap-regret learning, we employ the well-known reduction by \citet{blum2007from} with the $O(\log T)$ regret online learning algorithm for Brier loss, e.g., \ac{FTL}.
We present the algorithm and its proof in \Cref{app:bm}.

\begin{restatable}{lemma}{lembm}
	\label{lem:bm}
	There is an online algorithm $\mathsf{A}_\mathrm{BM}$ that, in each step $t \in [T]$, first predicts an $M \times M$ column-stochastic matrix $A_t$, and then observes outcome $y_t$ and a distribution $\pi_t \in \Delta(\calK^\varepsilon)$, such that for any transformation $\sigma: \simplex \rightarrow \simplex$,
	\begin{align*}
	\sum_{t \in [T]}\bbE_{p_t \sim A_t\pi_t}\ell(p_t,y_t) \leq \sum_{t \in [T]} \bbE_{p_t' \sim \pi_t} \ell(\sigma(p_t'), y_t) + O\left (M \log T + \varepsilon^2 T\right)
		~.
	\end{align*}
\end{restatable}
Intuitively, we may interpret the column-stochastic matrix $A_t$ from algorithm $\mathsf{A}_\mathrm{BM}$ as a suggested remapping from any prediction, so that for any sequence of outcomes $y_t$ and distributions of predictions $\pi_t$, the remapped/calibrated predictions are competitive against the best remapping $\sigma$
in hindsight.
The standard approach is then to sample a randomized prediction from the stationary distribution of 
$A_t$ (but we will do this step later after mixing $A_t$ with another remapping matrix).

\paragraph{Interpolating between calibration and multi-calibeating.}
Besides achieving small swap regret and calibration rate, we also want to follow the reference prediction $b_t$ from algorithm $\mathsf{A}^*$ to be competitive against this reference algorithm.
Observe that following the reference prediction corresponds to remapping every prediction to $b_t$, which can be captured by a remapping matrix $\smash{B_t = (b_t, b_t, \dots, b_t) \in \bbR^{M\times M}}$.
To hedge between these two factors, we resort to a lopsided two-expert algorithm $\mathsf{A}_\mathrm{lopsided}$ to obtain a weight $w_t \in [0, 1]$, and take a linear combination $C_t = w_t A_t + (1-w_t) B_t$ as the aggregated remapping. 

\begin{lemma}[\citealt{sani2014exploiting}]
	\label{lem:lopsided}
	There is an algorithm $\mathsf{A}_\mathrm{lopsided}$ for the expert problem with two experts, such that the expected regret w.r.t.\ expert $1$ is at most $O(\sqrt{T\log T})$, and the expected regret w.r.t.\ expert $2$ is at most $O(1)$.
\end{lemma}

Finally, we sample a prediction from the stationary distribution of $C_t$, as shown in \Cref{alg:calibeating-calibration}.

\begin{algorithm}[!htbp]
\caption{Multi-Calibeating + Calibration}
\label{alg:calibeating-calibration}
\textbf{Sub-routines:}\\[.5ex]
\begin{itemize}[after=\vspace{-0.5\baselineskip},parsep=0pt]
\setlength\itemsep{0pt}
\item Discretization and rounding algorithm $\mathsf{H}$ \citep[see][]{fishelson2025full}
\item Reference algorithm $\mathsf{A}^*$ (\Cref{alg:calibeating-from-regret} for calibeating, \Cref{alg:bestN-general} for multi-calibeating).
\item \ac{BM} reduction $\mathsf{A}_\mathrm{BM}$ (\Cref{alg:bm}).
\item Lopsided two-expert algorithm $\mathsf{A}_\mathrm{lopsided}$ (\Cref{alg:lopsided}).
\end{itemize}
\For{$t = 1$ \KwTo $T$}{
    \tcp{prediction}
    Algorithm $\mathsf{A}_\mathrm{BM}$ predicts $A_t$.\\
    Round algorithm $\mathsf{A}^*$'s prediction with $\mathsf{H}$ to get $b_t \in \Delta(\calK^\varepsilon)$, and let $B_t = (b_t, \dots, b_t)$.\\
    Algorithm $\mathsf{A}_\mathrm{lopsided}$ predicts $w_t \in [0, 1]$.\\
    Let $C_t = w_t A_t + (1-w_t) B_t$, and $\pi_t \in \Delta_M$ be its stationary distribution, i.e., $\pi_t = C_t\pi_t$.  \\
    Predict $p_t \sim \pi_t$. \\[0.5ex]
    \tcp{update}
    Observe outcome $y_t$. \\
    Update $\mathsf{A}_\mathrm{BM}$ and $\mathsf{A}^*$ based on $y_t$ and $\pi_t$ (applicable to the former).\\
    Update $\mathsf{A}_\mathrm{lopsided}$ with $\E_{z \sim A_t\pi_t} \ell(z, y_t)$ and $\E_{z \sim b_t} \ell(z, y_t)$ as the losses of experts 1 and 2.\\
    }
\end{algorithm}

\subsection{Analysis: Proof of a weaker version of \Cref{thm:calibeating-calibration}}
\label{sec:weaker-proof-calibeating-calibration}

We will prove a weaker guarantee of pseudo-calibration due to space constraints and defer the rest of the proof to \Cref{app:concentration}. 
By definition, the expected cumulative loss of \Cref{alg:calibeating-calibration} is

\[	
	\E\, L_T(p_{1:T},y_{1:T}) = \E_{\pi_{1:T}} \bigg[ \sum_{t \in [T]} \E_{z_t \sim \pi_t} \ell(z_t, y_t) \bigg]
	~.
\]
Since $\pi_t$ is the stationary distribution of $C_t = w_t A_t + (1-w_t) B_t$, the above further equals
\begin{align*}
	\sum_{t \in [T]} \E_{z_t \sim C_t\pi_t} \ell(z_t, y_t) 
	&
	= \sum_{t \in [T]} \bigg( w_t \, \E_{z_t \sim A_t\pi_t} \ell(z_t, y_t) + (1-w_t) \, \E_{z_t \sim B_t\pi_t} \ell(z_t, y_t) \bigg) \\
	&
	= \sum_{t \in [T]} \bigg( w_t \, \E_{z_t \sim A_t\pi_t} \ell(z_t, y_t) + (1-w_t) \, \E_{z_t \sim b_t}\ell(z_t, y_t) \bigg)
	~.
\end{align*}
By construction, $\E_{z_t \sim A_t\pi_t} \ell(z_t, y_t)$ and $\E_{z_t \sim b_t} \ell(z_t, y_t)$ are the losses of the two-expert problem in round $t \in [T]$.
The lopsided regret bounds of $\mathsf{A}_\mathrm{lopsided}$ (\Cref{lem:lopsided}) give
\begin{align}
	\E\, L_T(p_{1:T},y_{1:T}) & \le \E_{b_{1:T}} \bigg[ \sum_{t \in [T]} \E_{z_t \sim A_t \pi_t} \ell(z_t, y_t) \bigg] + O(\sqrt{T\log T})
	~,
	\label{eqn:expert-1} \\
	\E\, L_T(p_{1:T},y_{1:T}) & \le \E_{b_{1:T}} \bigg[ \sum_{t \in [T]} \E_{z_t \sim b_t} \ell(z_t, y_t) \bigg] + O(1)
	~.
	\label{eqn:expert-2}
\end{align}

\paragraph{Regret w.r.t.\ $\mathsf{A}^*$.}
Recall that $b_t$ is obtained by rounding the prediction from $\mathsf{A}^*$ with rounding algorithm $\mathsf{H}$.
By the $O(\varepsilon^2)$ rounding error bound of $\mathsf{H}$ (\Cref{lem:discretization}) and \Cref{eqn:expert-2}, the regret w.r.t.\ the reference algorithm $\mathsf{A}^*$ is at most $O(\varepsilon^2 T)$.

\paragraph{Pseudo-calibration.}
We will prove a weaker guarantee that
\[
	\E\, L_T(p_{1:T},y_{1:T}) \le \min_{\sigma: \simplex \rightarrow \simplex }\bbE \sum_{t=1}^T\ell(\sigma(p_t),y_t) + O \bigg( \sqrt{T \log T} + \frac{\sqrt{K}}{\varepsilon^{K-1}} \log T + \varepsilon^2 T \bigg)
	~.
\]

This follows from \Cref{eqn:expert-1}, the guarantee of \ac{BM} reduction (\Cref{lem:bm}), and that $M = O(\sqrt{K} \varepsilon^{-K+1})$ (\Cref{lem:discretization}).
It is weaker than the original statement as the choice of $\sigma$ does not depend on the realization of randomness of the algorithm.
By contrast, the original statement allows choosing $\sigma$ based on the realization of randomness.
We defer this concentration argument to \Cref{app:concentration}.

\section{Conclusion}

\label{sec:conclusion}

We have revisited calibeating through the lens of online learning and developed a reduction-based framework that connects calibeating and its extensions to standard online-learning notions. This  
viewpoint 
enables us to recover and sharpen existing results, extend them to general proper scoring losses, and deliver new matching lower bounds, 
in a unified and modular way. 
A natural direction for future work is to 
push this approach further—both to identify additional achievable guarantees and to improve online forecasting more broadly.
Another open question is whether one can simultaneously achieve an $O(|Q|\log T)$ calibeating rate and $\smash{\tilde{O}(T^{\frac{1}{3}})}$ calibration for the binary case, matching the best known bounds for each objective, and whether analogous simultaneous guarantees 
hold
beyond the Brier loss. In fact, the meta-algorithm developed in \Cref{sec:simultaneous-calibration-calibeating} suggests a general recipe that may apply more broadly. Whenever one can identify an appropriate discretization of the prediction space, a compatible rounding scheme with controlled loss guarantees, and a mechanism for upgrading pseudo-swap regret to true swap regret, the same framework should yield analogous extensions to other loss functions.

\appendix

\section{Omitted Proofs in \Cref{sec: calibeating}}

\label{app:calibeating}

\subsection{Proof of \Cref{thm:lower bound}}
\label{app:lower-bound}

\thmlowerbound*

\begin{proof}
We prove a general result that, for any integers $T_q\geq T_0$ with $\sum_{q\in Q} T_q=T $, we have, 
\begin{align*}
\inf _{\mathsf{A}} \sup _{\left(q_{1:T}, y_{1:T}\right)} \mathbb{E}_{p_{1:T}\sim \mathsf{A}}\left[L_T(p_{1:T},y_{1:T}) - R_T(q_{1:T},y_{1:T}) \right] \geq  \sum_{q\in Q} \beta\left(T_q\right).
\end{align*}

By Yao's minimax principle, \eqref{eq:lower-bound-regret} implies that, for any $T$,
there exists a distribution $\mathrm{S}_T \in \Delta(\calE^T)$ such that
\begin{align}
\label{eq:yao}	
 \min_{\mathsf{D}} \E_{y_{1:T}\sim \mathrm{S}_T}\left[ \sum_{t=1}^T \ell\left(p_t, y_t\right)-\min _{p \in \simplex} \sum_{t=1}^T \ell\left(p, y_t\right) \right] \geq \beta(T),
\end{align}
where the minimum is over deterministic online algorithms $\mathsf{D}$. 

For any $q$ and the corresponding $T_q$, let $\mathrm{S}_q = \mathrm{S}_{T_q}$ be the distribution guaranteed by \eqref{eq:yao} at horizon $T_q$.
Define the following distribution $\mathrm{S}$ over pairs$(q_{1:T}, y_{1:T})\in (Q\times \calE)^T$ as follows:
\begin{enumerate}
    \item Choose disjoint index sets $\{\calI_q\}_{q\in Q}$ with $\calI_q\subseteq [T]$, $|\calI_q| = T_q$ and $\cup_{q\in Q}\calI_q = [T]$. Set $q_t = q$ for all $t \in \calI_q$. 
    \item For each $q \in Q$, denote the subsequence of outcomes $y_t$s that $t\in \calI_q$ to be $y_{\calI_q} = (y_t)_{t \in \calI_q}$. Independently sample $y_{\calI_q}$ according to $\mathrm{S}_q$.
\end{enumerate}
Then by definition and the additivity of sums of expectations,
\begin{align*}
\MoveEqLeft{\min_{\mathsf{D}} \E_{(q_{1:T},y_{1:T})\sim \mathrm{S}}\left[ L_T(p_{1:T},y_{1:T}) -  R_T(q_{1:T},y_{1:T}) \right]} \\
&=  \min_{\mathsf{D}} \E_{(q_{1:T},y_{1:T})\sim \mathrm{S}}\left[ \sum_{q\in Q} \left(\sum_{t:q_t=q}\ell(p_t,y_t) - \min_{p\in\simplex}\sum_{t:q_t=q} \ell(p,y_t)\right)\right] \\
&= \min_{\mathsf{D}} \sum_{q\in Q}\E_{y_{\calI_q}\sim S_q}\left[\sum_{t:q_t=q}\ell(p_t,y_t) - \min_{p\in\simplex}\sum_{t:q_t=q} \ell(p,y_t) \right]\\
& \geq \sum_{q\in Q} \min_{\mathsf{D}}\E_{y_{\calI_q}\sim S_q}\left[\sum_{t:q_t=q}\ell(p_t,y_t) - \min_{p\in\simplex} \sum_{t:q_t=q} \ell(p,y_t) \right]\\
& \geq  \sum_{q\in Q}\beta(T_q).
\end{align*}
Therefore, since a randomized algorithm $A$ is a probability distribution over deterministic algorithms, it holds that
\begin{align*}
    \MoveEqLeft{ \min_{\mathsf{A}} \sup_{\left(q_{1:T},y_{1:T}\right)} \E_{p_{1:T}\sim \mathsf{A}}\left[L_T(p_{1:T},y_{1:T}) -  R_T(q_{1:T},y_{1:T}) \right ]} \\
     & \geq \min_{\mathsf{A}} \E_{(q_{1:T},y_{1:T})\sim \mathrm{S}}\left[ L_T(p_{1:T},y_{1:T}) -  R_T(q_{1:T},y_{1:T})\right] \geq \sum_{q\in Q}\beta(T_q)~.
\end{align*}
Choosing balanced $T_q \in\{\lfloor T / |Q|\rfloor,\lceil T / |Q| \rceil\}$ completes the proof.
\end{proof}

\section{Omitted Proofs in \Cref{sec: multi-calibeating}}

\label{app:multicalibeating}

\subsection{Proof of \Cref{thm:lower bound-multi-calibeating}}

\thmlowerboundmulticalibeating*

\begin{proof}
First, for any realization of expert predictions $\{p^{(n)}_t\}_{t\in[T],n\in[N]}$ and outcomes $y_{1:T}$, consider the multi-calibeaing problem with the same outcome sequence and the external forecasts $q^{(n)}_t=p^{(n)}_t$ for any $t\in [T]$ and $n\in [N]$ (Recall that the experts problem and multi-calibeating only differ in the benchmarks). By \Cref{lmm:operational-form},
\begin{align*}
 R_{T}(q^{(n)}_{1:T},y_{1:T})
=
\sum_{p}\min_{u\in\Delta_K}\sum_{t\le T:\,q_t^{(n)}=p}\ell(u,y_t)
 \le 
\sum_{p}\sum_{t\le T:\,q_t^{(n)}=p}\ell(p,y_t)
=
\sum_{t=1}^{T}\ell(q_t^{(n)},y_t),	
\end{align*}
where the inequality chooses $u=p$ in each bin. Taking $\min_{n\in[N]}$ on both sides yields
\begin{align*}
\min_{n\in[N]} R_{T}(q^{(n)}_{1:T},y_{1:T})
 \le\
\min_{n\in[N]}\sum_{t=1}^{T}\ell(q_t^{(n)},y_t).
\end{align*}
Therefore, for any predictions $p_{1:T}$,
\begin{align*}
L_T(p_{1:T},y_{1:T}) - \min_{n\in[N]}R_T(q^{(n)}_{1:T},y_{1:T}) &\ge \sum_{t=1}^{T}\ell(p_t,y_t) - \min_{n\in[N]}\sum_{t=1}^{T}\ell(q_t^{(n)},y_t)\\
&=
\sum_{t=1}^{T}\ell(p_t,y_t) - \min_{n\in[N]} \sum_{t=1}^{T}\ell(p^{(n)}_t, y_t),
\end{align*} 
and multi-calibeating inherits the lower bound of the expert problem. 

Second, consider the instances where $q^{(n)}_t = q_t$ for all $t\in [T]$, which is equivalent to beating only one external forecaster. Therefore, multi-calibeating inherits the lower bound of the calibeating problem. This completes the proof.

\end{proof}

\section{Omitted Proofs in \Cref{sec:simultaneous-calibration-calibeating}}
\label{app:calibeating-calibration}

\subsection{Proof of \Cref{coro:multi-class-calibeating-calibration-tradeoff}}

\coromulticlasscalibeatingcalibrationtradeoff*

\begin{proof}
Let $\varepsilon^2 T = T^x$, then the calibeating rate is $O(|Q^{(n)}|\log T + T^x)$ and $\varepsilon = T^{\frac{x-1}{2}}$. The corresponding calibration error is of the order $O(\sqrt{T\log T} + T^{\frac{(K-1)(1-x)}{2}}\log T + T^x)$, which always has greater order than the regret. As the order of calibration error is optimized when $x = \frac{(K-1)(1-x)}{2} = \frac{K-1}{K+1}$, we consider smaller $x$ with higher $T^{\frac{(K-1)(1-x)}{2}}$ term. The tradeoff follows by solving $T^{\frac{(K-1)(1-x)}{2}} < T$.
\end{proof}

\subsection{Algorithm $\mathsf{A}_\mathrm{BM}$ and the proof of \Cref{lem:bm}}
\label{app:bm}

\begin{subroutine}
	\caption{$\mathsf{A}_{\mathrm{BM}}$}
	\label{alg:bm}
	\textbf{Sub-routines:}\\[.5ex]
\begin{itemize}
\setlength\itemsep{0pt}
\item For each grid action $z_i$, a separate online learner $\mathsf{A}^{(i)}$ (e.g., \ac{FTL})
\item Discretization and rounding algorithm $\mathsf{H}$ \citep[see][]{fishelson2025full}
\end{itemize}
	\For{$t=1$ \KwTo $T$}{
        \tcp{prediction}
        \For{$i = 1$ \KwTo $M$}{
		Observe strategy $q^{(i)}_{t}\in \simplex$ from Algorithm $\mathsf{A}^{(i)}$. \\
		Round $q^{(i)}_{t}$ to $\mathsf{H}(q^{(i)}_t) \in \Delta(\calK^\varepsilon)$. \\
        }
		Output $A_t = (\mathsf{H}(q^{(1)}_t),\dots, \mathsf{H}(q^{(M)}_t)) \in \bbR^{M\times M}$.\\[1.5ex]
        \tcp{update}
		Receive feedback tuple $(y_t, \pi_t)$.\\
		Update Algorithm $\mathsf{A}^{(i)}$ with loss function $\pi_t(i)\ell(\cdot, y_t)$.\\
	}
\end{subroutine}

\lembm*

\begin{proof}
Denote $\calK^\varepsilon \coloneqq \{z_1,\dots, z_M\}$. For any $j \in [M]$ and $\sigma(z_j)\in \simplex$, we have
	\begin{align*}
	\MoveEqLeft{\sum_{t} \pi_{t}(j)\sum_{i}A_{t}(i,j)\ell(z_i,y_t) - \sum_{t}\pi_{t}(j)\ell(\sigma(z_j),y_t)}\\
	&= \sum_{t}\pi_{t}(j)\bbE_{i\sim \mathsf{H}(q^{(j)}_t)} \ell(z_i,y_t) - \sum_{t}\pi_{t}(j) \ell(\sigma(z_j),y_t)\\
	&\overset{(a)}{\leq} \sum_{t}\pi_{t}(j) \left( \ell(q^{(j)}_t,y_t) + C_2 \varepsilon^2\right) - \sum_{t}\pi_{t}(j)\ell(\sigma(z_j),y_t)\\
		& = \sum_{t} \pi_{t}(j) \ell(q^{(j)}_t, y_t) - \sum_{t}\pi_{t}(j)\ell(\sigma(z_j),y_t) + C_2\varepsilon^2\sum_{t}\pi_{t}(j)\\
	& \overset{(b)}{\leq} C_1 \log T + C_2\varepsilon^2\sum_{t}\pi_t(j),
	\end{align*}
	for some positive constants $C_1, C_2 >0$, where $(a)$ follows from the definition of the rounding scheme and the upper bound of the rounding error (\Cref{lem:discretization});
	$(b)$ holds by that \ac{FTL} has a regret of $O(\log T)$ under the Brier loss. 
	
	Summing this inequality over all $j \in [M]$, we obtain
	\begin{align*} \MoveEqLeft{\sum^T_{t=1}\bbE_{i\sim A_t\pi_t}\ell(z_i,y_t) - \sum^T_{t=1}\bbE_{j \sim \pi_t} \ell(\sigma(z_j))}\\
	& = \sum_{t=1}^T\sum_{j=1}^M \pi_t(j)\sum_{i=1}^MA_t(i,j)\ell(z_i, y_t) - \sum_{j=1}^M \pi_{t}(j) \ell(\sigma(z_j),y_t)\\
	& =  M C_1 \log T + C_2 \varepsilon^2\sum^T_{t=1}\sum^M_{j=1}\pi_{t}(j)\\
	& \leq \max\{C_1, C_2\} \left (\frac{\sqrt{K}\log T}{\varepsilon^{K-1}} + \varepsilon^2 T \right),
	\end{align*}
    where in the last inequality, we use  $M = O(\frac{\sqrt{K}}{\varepsilon^{K-1}})$ in \Cref{lem:discretization}. 
\end{proof}

\subsection{Algorithm $\mathsf{A}_{\mathrm{lopsided}}$}

\begin{subroutine}[!htbp]
	\caption{$\mathsf{A}_{\mathrm{lopsided}}$\citep{sani2014exploiting}}
	\label{alg:lopsided}
	\KwIn{learning rate $\eta \in (0,\frac{1}{2}]$, initial weights $\{s_1, 1-s_1\}$}
	\For{$t=1$ \KwTo $T$}{
        \tcp{prediction}
		Output weight $w_t = \frac{s_t}{s_t+1-s_1}$.\\[1.5ex]
        \tcp{update}
		Receive loss $\E_{z \sim A_t\pi_t} \ell(z, y_t)$ and $\E_{z \sim b_t} \ell(z, y_t)$ as the losses of experts 1 and 2, respectively. \\
		Compute $\delta_t = g_t^{(2)} - g_t^{(1)}$ and set 
		$s_{t+1}=s_t\cdot (1+\eta \delta_t)$. 
	}
\end{subroutine}

\subsection{Concentration arguments to finish the proof of \Cref{thm:calibeating-calibration}}
\label{app:concentration}

We finish the proof of \Cref{thm:calibeating-calibration} by upper-bounding the true calibration error using bounds of the pseudo-calibration error. 
For convenience, 
denote the pseudo-calibration error under an algorithm $\mathsf{A}$ to be $\widetilde{K}_T\coloneqq \min_{\sigma:\simplex\rightarrow \simplex}\E_{p_t \sim \mathsf{A}}\left[\sum^T_{t=1}\ell(p_t, y_t ) - \sum^T_{t=1}\ell(\sigma(p_t),y_t)\right ]$.
The following lemma is a multiclass extension of Theorem~3 in~\citet{luo2025simultaneous}, relating $\widetilde{K}_T$ to the calibration error.

\begin{restatable}{lemma}{lmmmulticlassexantetoexpost}
\label{lmm:multi-class-exante-to-expost}
For the Brier loss, for discretization with size $M$ and an algorithm $\mathsf{A}$ that always predicts the discretization grid points, with probability at least $1-\delta$ over the randomness in $\mathsf{A}$'s predictions $p_1, \dots, p_T$, we have 
\begin{align*}
K_T (p_{1:T}, y_{1:T}) \leq 6 \widetilde{K}_T+ 96 K M\log \frac{4KM}{\delta}.	
\end{align*}
\end{restatable}
Therefore, together with the weaker version proved in \Cref{sec:weaker-proof-calibeating-calibration} and $M = O(\frac{1}{\varepsilon^{K-1}})$, we have
\begin{align*}
K_T(p_{1:T},y_{1:T}) &\leq O \bigg( \sqrt{T \log T} + \frac{1}{\varepsilon^{K-1}} \left(K^{1/2}\log T + K^{3/2}\log\frac{4K^{3/2}}{\varepsilon^{K-1}}\right) + \varepsilon^2 T -\log \delta\bigg)\\
 &= O_{K,\log T}\bigg(\sqrt{T}+\frac{1}{\varepsilon^{K-1}}\log \frac{1}{\varepsilon} + \varepsilon^2T\bigg),
\end{align*}
with probability at least $1-\delta$.

\subsection{Proof of \Cref{lmm:multi-class-exante-to-expost}}

We first note the following fact on the closed-form of $K_T$ and $\tilde{K}_T$.

\begin{fact}
\label{lmm:close-form-calibration}
For the Brier losses, the calibration error under prediction sequence $p_{1:T}	$ and outcome sequence $y_{1:T}$ is 
\begin{align*}
	K_T = \sum_{p\in \simplex}\sum^T_{t=1}\mathbf{1}\{p_t = p\}\|p-\rho^p_T\|_2^2 = \sum_{p\in \simplex}\sum^T_{t=1}\mathbf{1}\{p_t = p\}\sum^K_{k=1}\left(p(k)-\rho_T^p(k)\right)^2,
\end{align*}
where $\rho^p_T = \sum_{t:p_t=p}\frac{y_t}{n_T(p)}$. 

The pseudo-calibration error under algorithm $A$ and outcome sequence $y_{1:T}$ is
\begin{align*}
\widetilde{K}_T = \sum^T_{t=1}\E_{p\sim P_t}\left[\|p-\tilde{\rho_T(p)}\|^2\right] = 	\sum^T_{t=1} \sum_{p}P_t(p)\sum^K_{k=1}\left (p(k)-\tilde{\rho}^p_T(k)\right)^2,
\end{align*}
where $P_t(p)$ is the randomized prediction under algorithm $A$ and $\tilde{\rho}^p_{T}\coloneqq \frac{\sum^T_{t=1}y_tP_t(p)}{\sum^T_{t=1}P_t(p)}$. 

\end{fact}

Denote the $i$th grid point by $z_i$ in the discretization, and $\rho^p_T(k)$ as $\rho^i(k)$ for simplicity. 
For any $k\in [K]$, let 
\begin{align*}
    K_T(k) \coloneq \sum_{i \in [M]} \sum^T_{t=1}\mathbf{1}\{p_t = z_i\}\left(z_i(k) - \rho^i(k)\right)^2,
\end{align*}
and
\begin{align*}
    \widetilde{K}_T(k) \coloneqq \sum_{i\in [M]}\sum^T_{t=1}P_t(z_i)(z_i(k)-\tilde{\rho}^i(k))^2.
\end{align*} 

We have the following lemma.
\begin{restatable}{lemma}{lmmsingleclasscalibrationconversion}
\label{lmm:single-class-calibration-conversion}
With probability at least $1-\delta$, $K_T(k) \leq 6 \widetilde{K}_T(k) + 96M \log \frac{4M}{\delta}$. 
\end{restatable}
Then, with a union bound across $k\in[K]$, with probability at least $1-\delta$, we have
\begin{align*}
	K_T = \sum^K_{k=1}K_T(k) \leq \sum^K_{k=1}\left(6\widetilde{K}_T(k) + 96 M \log \frac{4MK}{\delta}\right) = 6\widetilde{K}_T + 96 MK \log \frac{4MK}{\delta}.
\end{align*}

\subsection{Proof of \Cref{lmm:single-class-calibration-conversion}}

\lmmsingleclasscalibrationconversion*

While the proof of \Cref{lmm:single-class-calibration-conversion} follows almost exactly as the proof of Theorem 3 in \citet{luo2025simultaneous}, we include it here for completeness. 

\begin{proof}[Proof of \Cref{lmm:single-class-calibration-conversion}]
The proof relies on the following version of Freedman's inequality.

\begin{lemma}[\citealt{beygelzimer2011contextual}]
\label{lmm:freedman}
	Let $\{X_i\}^n_{i=1}$ be a martingale difference sequence adapted to the filtration $\mathcal{F}_1 \subseteq \cdots \subseteq \mathcal{F}_n$, where $\left|X_i\right| \leq B$ for all $i\in [n]$, and $B$ is a fixed constant. Define $\mathcal{V}:=\sum_{i=1}^n \mathbb{E}\left[X_i^2 \mid \mathcal{F}_{i-1}\right]$. Then, for any fixed $\mu \in\left[0, \frac{1}{B}\right], \delta \in[0,1]$, with probability at least $1-\delta$, we have
\begin{align*}
\left|\sum_{i=1}^n X_i\right| \leq \mu \mathcal{V}+\frac{\log \frac{2}{\delta}}{\mu} .
\end{align*}
\end{lemma}

Fix $i \in [M]$ and define the martingale difference sequence $X_t \coloneqq y_t(k) (P_t(i)-\mathbf{1}\{p_t = z_i\})$ and $Y_t \coloneqq P_t(i) - \mathbf{1}\{p_t = z_i\}$. Observe that $|X_t|\leq 1$, $|Y_t|\leq 1$ for all $t$. Fix $\mu_i \in [0,1]$. Applying \Cref{lmm:freedman} to the sequence $X\coloneqq Y\coloneqq X_{1:T}$, $Y_{1:T}$ and taking a union bound over them, we obtain that with probability at least $1-\delta$,
\begin{align}
\label{eq:bound-X-Y-mu-V} 
\left|\sum_{t=1}^T y_t(k)\left(P_t(i)-\mathbf{1}\left\{p_t=z_i\right\}\right)\right| \leq \mu_i \calV_X+\frac{\log \frac{4}{\delta}}{\mu_i},~ \left|\sum_{t=1}^T P_t(i)-\mathbf{1}\left\{p_t=z_i\right\}\right| \leq \mu_i \calV_Y+\frac{\log \frac{4}{\delta}}{\mu_i},	
\end{align}
where $\calV_X$, $\calV_Y$ are given by 
\begin{align*}
& \mathcal{V}_X=\sum_{t=1}^T \mathbb{E}\left[X_t^2 \mid \mathcal{F}_{t-1}\right]=\sum_{t=1}^T y_t(k) \cdot P_{t}(i)\left(1-P_{t}(i)\right) \leq \sum_{t=1}^T P_{t}(i), \text { and } \\
& \mathcal{V}_Y=\sum_{t=1}^T \mathbb{E}\left[Y_t^2 \mid \mathcal{F}_{t-1}\right]=\sum_{t=1}^T P_{t}(i)\left(1-P_{t}(i)\right) \leq \sum_{t=1}^T P_{t}(i).	
\end{align*}
The upper tail $\rho^i(k)-\tilde{\rho}^i(k)$ can then be bounded in the following manner:

\begin{align*}
\rho^i(k)-\tilde{\rho}^i(k) & =\frac{\sum_{t=1}^T y_t(k) \mathbf{1}\left\{p_t=z_i\right\}}{\sum_{t=1}^T \mathbf{1}\{p_t = z_i\}}-\frac{\sum_{t=1}^T y_t(k) P_t(i)}{\sum_{t=1}^T P_t(i)} \\
& \overset{(a)}{\leq} \frac{\sum_{t=1}^T y_t(k) \mathbf{1}\{p_t = z_i\}}{\sum_{t=1}^T \mathbf{1}\{p_t = z_i\}}+\frac{\mu_i \sum_{t=1}^T P_t(i)+\frac{\log \frac{4}{\delta}}{\mu_i}-\sum_{t=1}^T y_t(k) \mathbf{1}\{p_t = z_i\}}{\sum_{t=1}^T P_t(i)} \\
& =\frac{\sum_{t=1}^T y_t(k)\mathbf{1}\{p_t = z_i\}}{\left(\sum_{t=1}^T \mathbf{1}\{p_t = z_i\}\right)\left(\sum_{t=1}^T P_t(i)\right)} \cdot\left(\sum_{t=1}^T P_t(i)-\sum^T_{t=1}\mathbf{1}\{p_t = z_i\}\right) \\
&\phantom{=\quad}+\frac{\mu_i \sum_{t=1}^T P_t(i)+\frac{\log \frac{4}{\delta}}{\mu_i}}{\sum_{t=1}^T P_t(i)} \\
& \overset{(b)}{\leq} \frac{\sum_{t=1}^T y_t(k) \mathbf{1}\{p_t = z_i\}}{\left(\sum_{t=1}^T \mathbf{1}\{p_t = z_i\}\right)\left(\sum_{t=1}^T P_t(i)\right)} \cdot\left(\mu_i \sum_{t=1}^T P_t(i)+\frac{\log \frac{4}{\delta}}{\mu_i}\right)\\
& \phantom{=\quad}+\frac{\mu_i \sum_{t=1}^T P_t(i)+\frac{\log \frac{4}{\delta}}{\mu_i}}{\sum_{t=1}^T P_t(i)} \\
& \overset{(c)}{\leq} 2 \mu_i+\frac{2 \log \frac{4}{\delta}}{\mu_i \sum_{t=1}^T P_t(i)},
\end{align*}
where (a) and (b) follow from \eqref{eq:bound-X-Y-mu-V}, and (c) follows by $y_t(k)\mathbf{1}\{p_t = z_i\} \leq \mathbf{1}\{p_t = z_i\}$. The lower tail can be bounded in an exact manner as
\begin{align*}
\tilde{\rho}^i(k)-\rho^i(k) & =\frac{\sum_{t=1}^T y_t(k) P_t(i)}{\sum_{t=1}^T P_t(i)}-\frac{\sum_{t=1}^T y_t(k) \mathbf{1}\{p_t = z_i\}}{\sum_{t=1}^T \mathbf{1}\{p_t = z_i\}} \\
& {\leq} \frac{\sum_{t=1}^T y_t(k) \mathbf{1}\{p_t = z_i\}+\mu_i \sum_{t=1}^T P_t(i)+\frac{\log \frac{4}{\delta}}{\mu_i}}{\sum_{t=1}^T P_t(i)}-\frac{\sum_{t=1}^T y_t(k) \mathbf{1}\{p_t = z_i\}}{\sum_{t=1}^T \mathbf{1}\{p_t = z_i\}} \\
& =\frac{\sum_{t=1}^T y_t(k) \mathbf{1}\{p_t = z_i\}}{\left(\sum_{t=1}^T P_t(i)\right)\left(\sum_{t=1}^T \mathbf{1}\{p_t = z_i\}\right)} \cdot\left(\sum_{t=1}^T \mathbf{1}\{p_t = z_i\}-\sum^T_{t=1}P_t(i)\right)\\
& \phantom{=\quad} +\frac{\mu_i \sum_{t=1}^T P_t(i)+\frac{\log \frac{4}{\delta}}{\mu_i}}{\sum_{t=1}^T P_t(i)} \\
& {\leq} \frac{\sum_{t=1}^T y_t(k) \mathbf{1}\{p_t = z_i\}}{\left(\sum_{t=1}^T \mathbf{1}\{p_t = z_i\}\right)\left(\sum_{t=1}^T P_t(i)\right)} \cdot\left(\mu_i \sum_{t=1}^T P_t(i)+\frac{\log \frac{4}{\delta}}{\mu_i}\right)\\
& \phantom{=\quad}+\frac{\mu_i \sum_{t=1}^T P_t(i)+\frac{\log \frac{4}{\delta}}{\mu_i}}{\sum_{t=1}^T P_t(i)} \\
& \leq 2 \mu_i+\frac{2 \log \frac{4}{\delta}}{\mu_i \sum_{t=1}^T P_t(i)} .	
\end{align*}

Combining both bounds, we have shown that for a fixed $\mu_i\in[0,1]$, $|\rho^i(k)-\tilde{\rho}^i(k)|\leq 2\mu_i + \frac{\log \frac{4}{\delta}}{\mu_i \sum^T_{t=1}P_t(i)}$ holds with probability at least $1-\delta$. Taking a union bound over all $i$, with probability $1-\delta$, the following holds simultaneously for all $i$,
\begin{align}
\left|\sum_{t=1}^T y_t(k)\left(P_{t}(i)-\mathbf{1}\{p_t = z_i\}\right)\right| & \leq \mu_i \sum_{t=1}^T P_{t}(i)+\frac{\log \frac{4M}{\delta}}{\mu_i},\\
\left|\sum_{t=1}^T P_{t}(i)-\mathbf{1}\{p_t = z_i\}\right| & \leq \mu_i \sum_{t=1}^T P_{t}(i)+\frac{\log \frac{4M}{\delta}}{\mu_i}, \label{eq:bound-Y-mu} \\
\left|\rho^i(k)-\tilde{\rho}^i(k)\right| & \leq 2 \mu_i+\frac{2 \log \frac{4M}{\delta}}{\mu_i \sum_{t=1}^T P_{t}(i)}. 	\label{eq:bound-rho-mu}
\end{align}
Consider the function $g(\mu):=\mu+\frac{a}{\mu}$, where $a \geq 0$ is a fixed constant. Clearly, $\min _{\mu \in[0,1]} g(\mu)=2 \sqrt{a}$ when $a \leq 1$, and $1+a$ otherwise. Minimizing the bound in \eqref{eq:bound-rho-mu} with respect to $\mu_i$, we obtain
\begin{align*}
\left|\rho^i(k)-\tilde{\rho}^i(k)\right| \leq \begin{cases} 4 \sqrt{\frac{\log \frac{4M}{\delta}}{\sum_{t=1}^T P_{t}(i)}}, \text { when } \log \frac{4M}{\delta} \leq \sum_{t=1}^T P_{t}(i),\\
   	2+\frac{2 \log \frac{4M}{\delta}}{\sum_{t=1}^T P_{t}(i)}, \text{ when } \log \frac{4M}{\delta}>\sum_{t=1}^T P_{t}(i). 
 \end{cases}
\end{align*}

Therefore, when $\sum_{t=1}^T P_{t}(i)$ is tiny, which is possible if algorithm $A$ does not allocate enough probability mass to the index $i$, the bound obtained is large making it much worse than the trivial bound $\left|\rho^i(k)-\tilde{\rho}^i(k)\right| \leq 1$ which follows since $\rho^i(k), \tilde{\rho}^i(k) \in[0,1]$ by definition. Based on this reasoning, we define the set
\begin{align}
\label{eq:def-calI}
	\calI:=\left\{i \in [M], \text { s.t. } \log \frac{4M}{\delta} \leq \sum_{t=1}^T P_{t}(i)\right\},
\end{align}
and let $\bar{\calI} \coloneqq [M]\setminus \calI$. We bound $(\rho^i(k)-\tilde{\rho}^i(k))^2$ as
\begin{align}
\label{eq:final-bound-rho}
\left(\rho^i(k)-\tilde{\rho}^i(k)\right)^2 \leq \begin{cases}\frac{16 \log \frac{4M}{\delta}}{\sum_{t=1}^T P_{t}(i)} & \text { if } i \in \calI, \\ 1 & \text { otherwise. }\end{cases}
\end{align}

Similarly, $\left|\sum_{t=1}^T P_{t}(i)-\mathbf{1}\{p_t = z_i\}\right|$ can be bounded by substituting the optimal $\mu_i$ obtained above in \eqref{eq:bound-Y-mu}; we obtain
\begin{align}
\label{eq:final-bound-Y}
	\left|\sum_{t=1}^T P_{t}(i)-\mathbf{1}\{p_t = z_i\}\right| \leq \begin{cases}2 \sqrt{\log \frac{4M}{\delta} \sum_{t=1}^T P_{t}(i)} & \text { if } i \in \calI \\ \sum_{t=1}^T P_{t}(i)+\log \frac{4M}{\delta} & \text { otherwise. }\end{cases}
\end{align}

Equipped with \eqref{eq:final-bound-rho} and \eqref{eq:final-bound-Y}, we proceed to bound $K_T(k)$ in the following manner:
\begin{align*}
K_T(k) &=\sum_{i \in [M]} \sum_{t=1}^T \mathbf{1}\{p_t = z_i\}\left(z_i(k)-\rho^i(k)\right)^2 \\
&\leq 2 \sum_{i\in [M]} \sum_{t=1}^T \mathbf{1}\{p_t = z_i\}\left(\left(z_i(k)-\tilde{\rho}^i(k)\right)^2+\left(\rho^i(k)-\tilde{\rho}^i(k)\right)^2\right),	
\end{align*}
where the inequality holds because $(a+b)^2 \leq 2 a^2+2 b^2$ for all $a, b \in \mathbb{R}$. To further bound the term above, we split the summation into two terms $\calT_1, \calT_2$ defined as
\begin{align*}
	\calT_1 & :=\sum_{i \in \calI} \sum_{t=1}^T \mathbf{1}\{p_t = z_i\}\left(\left(z_i(k)-\tilde{\rho}^i(k)\right)^2+\left(\rho^i(k)-\tilde{\rho}^i(k)\right)^2\right), \\
\calT_2 & =\sum_{i \in \bar{\calI}} \sum_{t=1}^T \mathbf{1}\{p_t = z_i\}\left(\left(z_i(k)-\tilde{\rho}^i(k)\right)^2+\left(\rho^i(k)-\tilde{\rho}^i(k)\right)^2\right),
\end{align*}
and bound $\calT_1$ and $\calT_2$ individually. We bound $\calT_1$ as
\begin{align*}
	\calT_1 & \overset{(a)}{\leq} \sum_{i \in \calI}\left(\sum_{t=1}^T P_{t}(i)+2 \sqrt{\log \frac{4M}{\delta} \sum_{\tau=1}^T P_{\tau} (i)}\right)\left(\left(z_i(k)-\tilde{\rho}^i(k)\right)^2+\frac{16 \log \frac{4M}{\delta}}{\sum_{\tau=1}^T P_{\tau} (i)}\right) \\
& =\sum_{i \in \calI} \sum_{t=1}^T P_{t}(i)\left(z_i(k)-\tilde{\rho}^i(k)\right)^2+16 \log \frac{4M}{\delta}|\calI|+ \\
& \phantom{=\quad} 2 \sum_{i \in \calI} \sqrt{\log \frac{4M}{\delta} \sum_{\tau=1}^T P_{\tau} (i)}\left(\left(z_i(k)-\tilde{\rho}^i(k)\right)^2+\frac{16 \log \frac{4M}{\delta}}{\sum_{\tau=1}^T P_{\tau} (i)}\right) \\
& \overset{(b)}{\leq} \sum_{i \in \calI} \sum_{t=1}^T P_{t}(i)\left(z_i(k)-\tilde{\rho}^i(k)\right)^2+16 \log \frac{4M}{\delta}|\calI|+ \\
& \phantom{=\quad}2 \sum_{i \in \calI} \sum_{\tau=1}^T P_{\tau} (i)\left(\left(z_i(k)-\tilde{\rho}^i(k)\right)^2+\frac{16 \log \frac{4M}{\delta}}{\sum_{\tau=1}^T P_{\tau} (i)}\right) \\
& =3 \sum_{i \in \calI} \sum_{t=1}^T P_{t}(i)\left(z_i(k)-\tilde{\rho}^i(k)\right)^2+48 \log \frac{4M}{\delta}|\calI|
\end{align*}
where (a) follows by substituting the bounds from \eqref{eq:final-bound-rho} and \eqref{eq:final-bound-Y}; while (b) follows since by the definition of $\calI$ in \eqref{eq:def-calI}, we have $\sqrt{\log \frac{4M}{\delta} \sum_{\tau=1}^T P_{\tau} (i)} \leq \sum_{\tau=1}^T P_{\tau} (i)$. Next, we bound $\calT_2$ as
\begin{align*}
	\calT_2 & \overset{(a)}{\leq} \sum_{i \in \overline{\calI}}\left(2 \sum_{t=1}^T P_{t}(i)+\log \frac{4M}{\delta}\right)\left(\left(z_i(k)-\tilde{\rho}^i(k)\right)^2+1\right) \\
& \overset{(b)}{\leq} 2 \sum_{i \in \overline{\calI}} \sum_{t=1}^T P_{t}(i)\left(z_i(k)-\tilde{\rho}^i(k)\right)^2+2 \sum_{i \in \overline{\calI}} \sum_{t=1}^T P_{t}(i)+2 \log \frac{4M}{\delta}|\overline{\calI}| \\
& \overset{(c)}{\leq} 2 \sum_{i \in \overline{\calI}} \sum_{t=1}^T P_{t}(i)\left(z_i(k)-\tilde{\rho}^i(k)\right)^2+4 \log \frac{4M}{\delta}|\overline{\calI}|
\end{align*}
where (a) follows by substituting the bounds from \eqref{eq:final-bound-rho} and \eqref{eq:final-bound-Y}; (b) follows by bounding $(z_i(k)-\tilde{\rho}^i(k))^2 \leq 1$; and (c) follows from the definition of $\calI$ in \eqref{eq:def-calI}. Collecting the bounds on $\calT_1$ and $\calT_2$, we obtain
\begin{align*}
	\calT_1+\calT_2 & \leq 3 \sum_{i\in [M]} \sum_{t=1}^T P_{t}(i)\left(z_i(k)-\tilde{\rho}^i(k)\right)^2+48 \log \frac{4M}{\delta}|\calI|+4 \log \frac{4M}{\delta}|\overline{\calI}| \\
& \leq 3 \widetilde{K}_T(k)+48M \log \frac{4M}{\delta},
\end{align*}
where the last inequality follows from the definition of $\widetilde{K}_T(k)$ and since $|\calI|+|\overline{\calI}|=M$. Since $K_T(k) \leq 2\left(\calT_1+\calT_2\right)$, we have shown that
\begin{align*}
K_T(k) \leq 6 \widetilde{K}_T(k) +96M \log \frac{4M}{\delta},	
\end{align*}
with probability at least $1-\delta$. This completes the proof.
\end{proof}

\printbibliography

\end{document}